\def\year{2019}\relax
\documentclass[letterpaper]{article} 
\usepackage{aaai19}  
\usepackage{times}  
\usepackage{helvet}  
\usepackage{courier}  
\usepackage{url}  
\usepackage{graphicx}  

\usepackage{amsmath,amssymb,amsthm,dsfont,algorithm,algorithmic,subfig,xcolor}
\theoremstyle{plain}
\newtheorem{theorem}{Theorem} 

\newtheorem{corollary}[theorem]{Corollary}

\theoremstyle{definition}
\newtheorem{definition}[theorem]{Definition}

\theoremstyle{remark}
\newtheorem{remark}{Remark}

\DeclareMathOperator{\supp}{supp}
 
\begin{document}
%
\title{State-Augmentation Transformations for Risk-Sensitive Reinforcement Learning}
\author{Shuai Ma \ and Jia Yuan Yu\\
Concordia Institute of Information System Engineering, Concordia University\\
1455 De Maisonneuve Blvd. W., Montreal, Quebec, Canada H3G 1M8\\
m\_shua@encs.concordia.ca, jiayuan.yu@concordia.ca\\
}
\maketitle
\begin{abstract}
In the framework of MDP, although the general reward function takes three arguments---current state, action, and successor state; it is often simplified to a function of two arguments---current state and action. The former is called a \textit{transition-based} reward function, whereas the latter is called a \textit{state-based} reward function. 
When the objective involves the expected total reward only, this simplification works perfectly. However, when the objective is risk-sensitive, this simplification leads to an incorrect value. 
We propose three successively more general state-augmentation transformations (SATs), which preserve the reward sequences as well as the reward distributions and the optimal policy in risk-sensitive reinforcement learning. 
In risk-sensitive scenarios, firstly we prove that, for every MDP with a stochastic transition-based reward function, there exists an MDP with a deterministic state-based reward function, such that for any given (randomized) policy for the first MDP, there exists a corresponding policy for the second MDP, such that both Markov reward processes share the same reward sequence. 
Secondly we illustrate that two situations require the proposed SATs in an inventory control problem. One could be using Q-learning (or other learning methods) on MDPs with transition-based reward functions, and the other could be using methods, which are for the Markov processes with a deterministic state-based reward functions, on the Markov processes with general reward functions.
We show the advantage of the SATs by considering Value-at-Risk as an example, which is a risk measure on the reward distribution instead of the measures (such as mean and variance) of the distribution. 
We illustrate the error in the reward distribution estimation from the reward simplification, and show how the SATs enable a variance formula to work on Markov processes with general reward functions.

\end{abstract}

\section{Introduction}
The most widely used optimization criterion in reinforcement learning (RL) is represented by the expected total reward.
This risk-neutral criterion renders two functions---the value function and the Q-function---important in many applications.
The two functions can be considered as substitutes for the reward functions, 
even in risk-sensitive reinforcement learning~\cite{borkar2002q,shen2014risk,garcia2015comprehensive,junges2016safety,gilbert2016quantile,huang2017risk,chow2017risk,berkenkamp2017safe}.
However, the reward function is usually in a more complicated form. 
In a Markov decision process (MDP), we say the reward function is \textit{transition-based} if it involves current state, action, and successor state, and \textit{state-based} if it involves current state and action only.
Furthermore, the reward could be stochastic.
For a (stochastic) transition-based reward function, the value function, or Q-function, implies a reward simplification, which changes the reward sequence $ \{ R_t \} $, and only preserves the mean of the reward distribution. 
This may not be sufficient in many risk-sensitive applications, especially when the small probability events have serious consequences, such as self-driving and medical diagnosis.
That is where a risk-sensitive criterion should be considered.
A risk-sensitive criterion refers to a risk measure, or a risk function, which maps a random variable to a scalar.
Since the risk measure depends on the distribution of the random variable, the reward simplification will lead to an incorrect result.

The motivation for our proposed methods arises from a contradiction between theory and practice with respect to the reward function.
On the one hand, some methods require MDPs, or Markov reward processes, to be with deterministic (and state-based) reward functions.
On the other hand, for many practical problems, the underlying MDPs or Markov processes have stochastic (and transition-based) reward functions, and the reward simplification changes the reward sequences, as well as the reward distributions, which are crucial in risk-sensitive scenarios.
This paper proposes three successively more general state-augmentation transformations (SATs) for different settings (Cases 1, 2, and 3) to solve this problem.
With the aid of the SATs, we can apply methods, which requires a deterministic state-based reward function, to MDPs (or Markov processes) with stochastic transition-based reward function, and at the same time preserve the reward sequence (and the reward distribution). 
We study the return---the discounted total reward $\sum_{t=0}^{\infty} \gamma^{t-1} R_t $---in an infinite-horizon MDP with finite state and action spaces, and consider the Value-at-Risk (VaR) objective as a risk-sensitive example. 
We generalize the transformation in~\cite{Ma2017STT} to three successively more general SATs (Cases 1, 2, and 3), give a proof for the most general one, and illustrate the error from the reward simplification on the return distribution. 

In Section 2, we define the notations for the MDPs with four types of reward functions and two policy spaces, and pin down the reward simplification as the main problem in a risk-sensitive setting. We consider two VaR objectives to show the effect of the reward simplification, since VaRs are functions of the reward distribution instead of measures (moments, such as mean and variance) of the distribution.
An infinite-horizon MDP is defined for an inventory control problem, which is used as an example for the proposed transformations. 
The return variance formula requires the Markov process to be with a deterministic reward function.
In Section 3, we propose SATs for different cases, and show the error from the reward simplification.
In Section 4, we give a literature review on risk study in reinforcement learning and risk-aware Q-learning. 
In Section 5, we have a discussion on the proposed transformations. 

Briefly, in a policy evaluation setting, when the objective is risk-sensitive and the Markov reward process is with a (stochastic) transition-based reward function, the return distribution should be preserved instead of the expectation only, and an appropriate SAT should be applied.
In a control setting, when a randomized policy is considered, we can apply the SAT in Case 3 to preserve the reward sequence within a specific policy space.
For related studies which concerned risk-sensitive problems in RL, when the reward function is not deterministic and state-based, we believe that they should be revisited with proposed transformations. 

\section{Preliminaries and Notations}
In this section, firstly we present the notations for MDPs with four types of reward functions and two policy space. 
Secondly, we define two VaR objectives and the VaR function. 
Thirdly, we consider an inventory control problem, which is a straightforward example of MDP with a transition-based reward function.
\subsection{Markov Decision Processes}

In this paper we focus on infinite-horizon discrete-time MDPs, which can be represented by
\[
\langle S, A, r, p, \mu , \gamma \rangle,
\]
in which
$S $ is a finite state space, and $X_t \in S$ represents the state at (decision) epoch $t \in \mathbb{N}$;
$A_x$ is the allowable action set for $x \in S$, $A = \bigcup_{x \in S}A_x$ is a finite action space, and $K_t \in A$ represents the action at epoch $t$; 
$r$ is a bounded reward function, and $R_t$ denotes the immediate reward at epoch $t$;
$p(y \mid x, a) = \mathbb{P}(X_{t+1}=y \mid X_t = x, K_t = a)$ denotes the homogeneous transition probability;
$\mu$ is the initial state distribution; $ \gamma \in (0,1)$ is the discount factor.

In this paper we study the distribution of the return $ \sum_{t=1}^{\infty} \gamma^{t-1} R_t$ in infinite-horizon MDPs. 
For $ x, y \in S, a \in A_x $, here we consider four types of reward functions: the deterministic state-based reward 
\begin{equation*}
	r_{DS}(x,a) \in \mathbb{R};
	\label{rewardTypeDS}
\end{equation*} 
the deterministic transition-based reward
\begin{equation*}
	r_{DT}(x,a,y) \in \mathbb{R};
	\label{rewardTypeDT}
\end{equation*} 
the stochastic state-based reward 
\begin{equation*}
	r_{SS}(j \mid x,a)=\mathbb{P}(R_t = j \mid X_t=x, K_t= a) \in [0,1];
	\label{rewardTypeSS}
\end{equation*} 
and the stochastic transition-based reward 
\begin{equation}
\begin{aligned}
	&r_{ST}(j \mid x,a,y) \\
	&= \mathbb{P}(R_t = j \mid X_t=x, K_t= a, X_{t+1}=y) \in [0,1].
\end{aligned}
\label{rewardTypeST}
\end{equation}
With a slight abuse of notation, we also represent, for example, $ r_{ST}(j \mid x,y) = \mathbb{P}(R_t = j \mid X_t=x, X_{t+1}=y) \in [0,1] $ for a Markov reward process.

When the reward function is not $ r_{DS} $ type, it is often simplified in the expectation way. For example, given a $r_{DT}$, the reward function can be simplified to a $ r_{DS} $ by 
\begin{equation}
	r_{DS}(x,a) = \sum_{y \in S} p(y \mid x,a) r_{DT}(x,a,y),
	\label{simplifiedR}
\end{equation} 
where $ x,y \in S, a \in A_x $. In practical problems, stochastic reward functions are often naively simplified to $ r_{DS} $ functions in a similar way.
In RL, when the expected return is considered, and the Q-function or the value function is accessed, which implies such a reward simplification. 
The effect of the reward simplification on return distribution in a finite-horizon Markov reward process has been studied in~\cite{Ma2017STT}. 
Here we estimate the distribution with assuming it is approximately normal, illustrate the similar effect on return distribution, and generalize the transformation for more practical cases. 

A policy $\pi$ describes how to choose actions sequentially.
For infinite-horizon MDPs, we focus on two stationary Markovian policy spaces: the deterministic policy space $ \Pi_D $, and the randomized policy space $ \Pi_R $. A (time-homogeneous) Markov reward process is tantamount to an MDP with a (randomized) policy. 
Randomized policy is often considered in constrained MDPs~\cite{altman1999constrained}. Given an MDP with a randomized policy, the reward function is often naively simplified as well. 
Since most risk measures are law invariant~\cite{kusuoka2001law}, we generalize the transformation for settings mentioned above, in order to preserve the return distribution.

\subsection{Value-at-Risk}
Value-at-Risk originates from finance. For a given portfolio (which can be considered as an MDP with a policy), a loss threshold (target level), and a time-horizon, VaR concerns the probability that the loss on the portfolio exceeds the threshold over the time horizon. VaR is hard to deal with since it is not a coherent risk measure~\cite{Riedel2004}. 
Two VaR problems described in~\cite{Filar1995b} are considered as optional objectives. 
Given an initial distribution $\mu$ and a policy $\pi $ in a specified policy space $ \Pi $, denote the return by $\Phi^{\pi}_{\mu} $, and here we simplify it to $\Phi$. Denote the return distribution with the policy $\pi$ by $F^{\pi}_{\Phi}$. VaR addresses the following problems. 
\begin{definition}
	Given a quantile $\alpha \in [0,1]$, find the optimal threshold $\rho_{\alpha} =  \sup\{\tau \in \mathbb{R} \mid  \mathbb{P}(\Phi > \tau) \ge \alpha, \pi \in \Pi\}=\sup\{\tau \in \mathbb{R} \mid  F^{\pi}_{\Phi}(\tau) \le 1-\alpha, \pi \in \Pi\}$.
\end{definition}	
This problem refers to the quantile function, i.e., $F^{\pi -1}_{\Phi}$.
\begin{definition}
	Given a threshold $\tau \in \mathbb{R}$, find the optimal quantile $\eta_{\tau} = \sup\{\alpha \in [0,1] \mid  F^{\pi}_{\Phi}(\tau) \le 1-\alpha, \pi \in \Pi\}$.
	\label{VaR2}
\end{definition}
This problem concerns $F^{\pi}_{\Phi}$. 
When the estimated return distribution is strictly increasing, any point along the function $\inf\{F^{\pi}_{\Phi}  \mid  \pi \in \Pi\}$ is (estimated) $(\rho_{\alpha}, 1-\eta_{\tau})$ with $\tau = \rho_{\alpha}$ or $\alpha = 1-\eta_{\tau}$. Therefore, both VaR objectives refers to the infimum function, and here we call it the \textit{VaR function}. 
Since most risk measures are law invariant, we consider VaR as an example to show the effect on the distribution from the reward simplification. 

\subsection{An MDP for Inventory Control Problems} 
We constructs an MDP for a single-product stochastic inventory control problem based on~\cite[Section 3.2.1]{Puterman1994a}.
Define the inventory capacity $M \in \mathbb{N}^+$, and the state space $S=\{0, \cdots, M\}$. Briefly, at time epoch $t \in \mathbb{N}$, denote the inventory level by $X_t$ before the order, the order quantity by $K_t \in \{0, \cdots, M-X_t\}$, the demand by $D_t$ with a time-homogeneous probability distribution $\mathbb{P}(D_t=i)$, where $i \in \{0, \cdots, M\}$, then we have $X_{t+1}=\max\{X_t+K_t-D_t,0\}$. 

For $x \in S$, denote the cost to order $x$ units by $c(x)$, a fixed cost $W \ge 0$ for placing orders, then we have the order cost $o(x) = (W+c(x)) \mathds{1}_{[x>0]}$.
Denote the revenue when $x$ units of demand is fulfilled by $f(x)$, the maintenance fee by $ m(x) $. The real reward function is $r(X_t,K_t,X_{t+1}) = f(X_t+K_t-X_{t+1})-o(K_t) - m(X_t)$. 

We set the parameters as follows. The fixed order cost $W=4$, the variable order cost $c(x)=2x$, the maintenance fee $ m(x) = x $, the warehouse capacity $M=2$, and the price $f(x)=8x$. The probabilities of demands are $\mathbb{P}(D_t=0)=0.25$, $\mathbb{P}(D_t=1)=0.5$, $\mathbb{P}(D_t=2)=0.25$ respectively. The initial distribution $\mu(0) = 1$. 
In this infinite-horizon MDP, the reward function is deterministic and transition-based. The simplified reward function $r'$ can be calculated by Equation~(\ref{simplifiedR}), which is state-based. 
As illustrated in Figure~\ref{MdpDesc}, now we have two MDPs with different reward functions: $\langle S, A, r, p, \mu, \gamma \rangle$ and $\langle S, A, r', p, \mu, \gamma \rangle$. 

\begin{figure}[h] 
	\centering
	\includegraphics[scale=0.9]{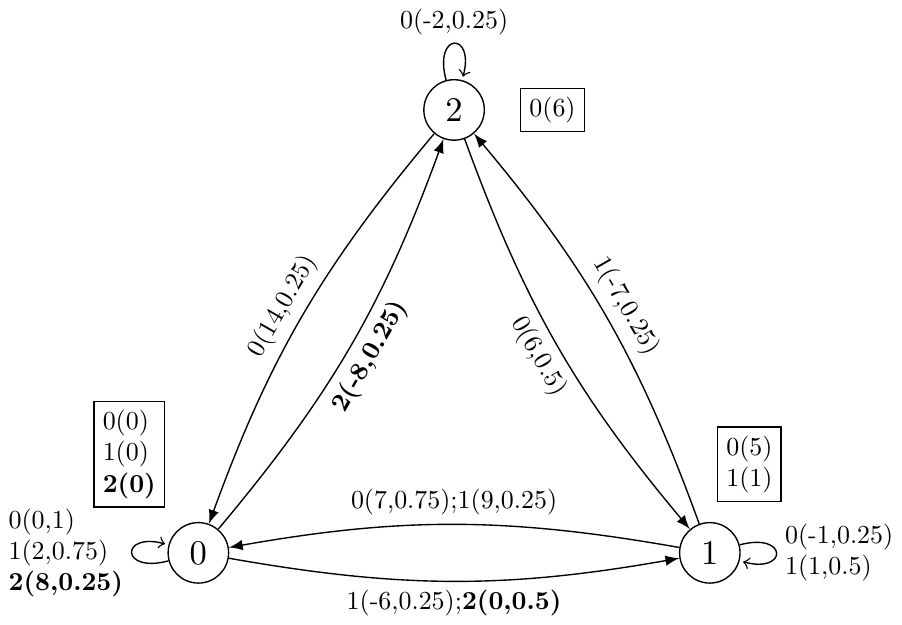} 
	\caption{An MDP with a transition-based reward function and its counterpart with a state-based reward function following reward simplication. 
Labels along transitions denote $a(r(x,a,y),p(y | x,a))$, and labels next to states denote \boxed{a(r'(x,a))}, the state-based reward function simplified with Equation~(\ref{simplifiedR}). 
For example, the labels in bold are interpreted as follows: the label $2(0,0.5)$ below the transition from 0 to 1 means that the reward $r(0,2,1)=0$ and the transition probability is 0.5; the label \boxed{2(0)} near state 0 means when $X_t = 0$ and $K_t=2$, the simplified reward $r'(0,2) = 0$.}  

	\label{MdpDesc}
\end{figure}

\subsection{Reward Distribution Estimation}
Since we consider risk from the distributional perspective instead of the expected return, we need to estimate the reward distribution.
The functional distribution estimations in Makrov reward processes have been studied for decades~\cite{woodroofe1992central,meyn2012markov}. 
However, there is no related central limit theorem for the discounted sum of rewards (return). 
To simplify the estimation, we estimate the return distribution by considering mean and variance only, which implies the assumption that the return is approximately normal. 
%
%
%
For an infinite-horizon Markov reward process with a deterministic state-based reward function, Sobel~\shortcite{sobel1982variance} presented the return variance formula for Markov reward processes.

\begin{theorem}~\cite{sobel1982variance} 
	Given an infinite-horizon Markov reward process $\langle S, r'_{\pi}, p_{\pi}, \gamma \rangle$ with the finite state space $ S = \{1, \cdots, |S|\} $, the reward function $r'_{\pi}$ deterministic state-based and bounded, and the discount factor $ \gamma \in (0,1)$. 
	Denote the transition matrix by $ P $, in which $ P(x,y) = p_{\pi}(y \mid x), x,y \in S $.
	Denote the conditional return expectation by $v_x = \mathbb{E}(\Phi  \mid  X_0 = x)$ for any deterministic initial state $x \in S$, and the conditional expectation vector by $v$. Similarly, denote the conditional return variance by $\psi_x = \mathbb{V}(\Phi \mid X_0 = x)$, and the conditional variance vector by $\psi$. 
	Let $\theta$ denote the vector whose $x$th component is $\theta_x = \sum_{y \in S}p_{\pi}(x,y) (r'_{\pi}(x) + \gamma v_y)^2 - v^2_x$. 
	Then
	\begin{gather*}
		v = r' + \gamma P v = (I - \gamma P)^{-1} r', \\ 
		\psi = \theta + \gamma^2 P \psi = (I - \gamma^2 P)^{-1} \theta.
	\end{gather*}
	\label{theorem1982}
\end{theorem}
Now with the aid of Theorem~\ref{theorem1982}, we can estimate the return distribution for the ergodic Markov reward process. Notice that the variance formula is for Markov reward process with a deterministic reward function only. 
In next section, we generalize the transformation for the Markov reward processes and MDPs in different cases, estimate the return distribution in the inventory control problem with the aid of a transformation, and compare it with the one from the reward simplification.

\section{State-Augmentation Transformations} 
In this section, we propose the state-augmentation transformations (SATs) for three cases. 
\begin{itemize}
	\item Case 1: a Markov reward process with a stochastic, transition-based reward function; 
	\item Case 2: an MDP with a stochastic transition-based reward function, and a randomized policy; and
	\item Case 3: an MDP with a randomized policy space.
\end{itemize}

Case 1 can be considered as an MDP with a $ r_{ST} $ (or $r_{SS} $) and a deterministic policy.  
Case 2 refers to the constrained MDPs. 
Case 3 describes a direct policy search (gradient descent method, for example) scenario from a risk-sensitive perspective.
In all the three cases, the reward functions are often simplified in a similar way as in Equation~(\ref{simplifiedR}), which will lose all moment information except for the first one (mean).
Noticing that the state-transition transformation~\cite{Ma2017STT} is for a Markov reward process with a deterministic, transition-based reward function, which we define as Case 0.
Since Case 0 is a special case of Case 1, we denote this relationship by $ \textit{Case 0} \prec \textit{Case 1} $. 
Similarly, the four cases have the relationship 
\begin{equation*}
\textit{Case 0} \prec \textit{Case 1} \prec \textit{Case 2} \prec \textit{Case 3}.
\label{CaseRelation}
\end{equation*}
Concisely, we review the original state-transition transformation in Algorithm~\ref{STT}, give a theorem (Theorem~\ref{TransMDP}) and its proof for the most general Case 3, and two corollaries (Corollary~\ref{tC2}, Corollary~\ref{tC1}) for Case 2 and 1. 
Considering the relationship between the cases, all needed algorithms and theorems for different cases can be derived from the constructive proof for Theorem~\ref{TransMDP} with some slight changes.

\begin{figure*}[t]
	\centering
	\setlength\fboxsep{0pt}
	\setlength\fboxrule{0.25pt}	\centering
	\subfloat[]{%
		\resizebox*{6.5cm}{!}{\includegraphics{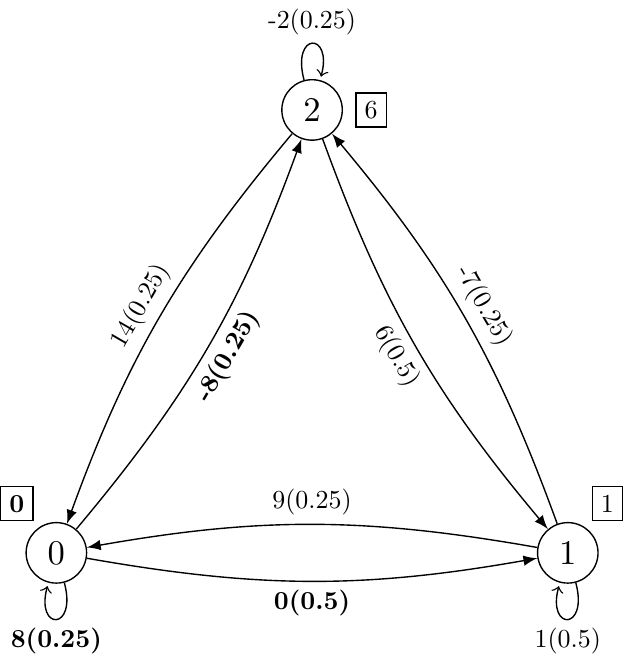}}}\hspace{50pt} 
	\label{Mrp0}
	\subfloat[]{%
		\resizebox*{7.2cm}{!}{\includegraphics{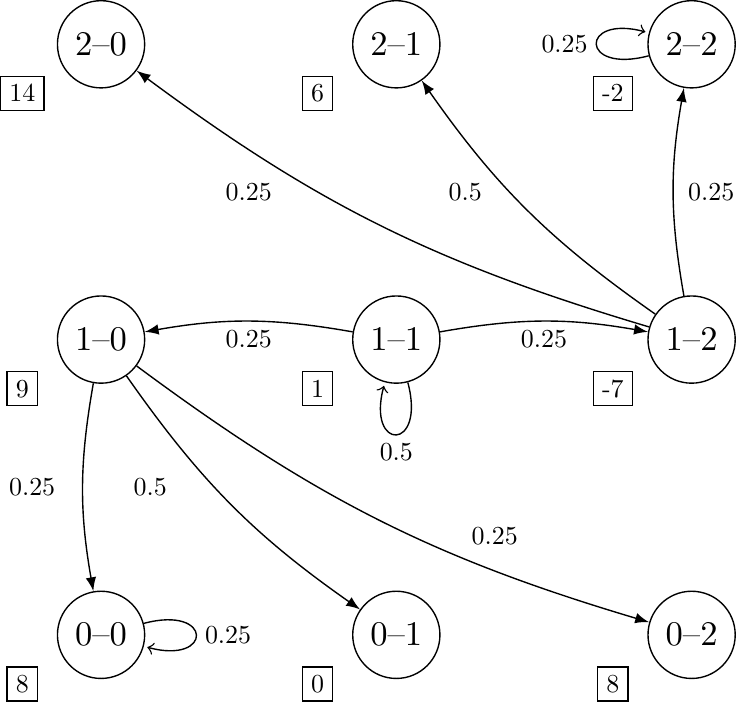}}} 
	\caption{(a) The Markov reward process for the MDP in Figure~\ref{MdpDesc} with the policy $ \pi $. This Markov process has a transition-based reward function. The labels along transitions denote $r_{\pi}(x,y)(p_{\pi}(x,y))$, and the labels $\boxed{r'_{\pi}(x)}$ near states denote the state-based reward function simplified with Equation~(\ref{simplifiedR}); (b) The transformed Markov reward process with a state-based reward function. For a Markov reward process with a deterministic transition-based reward function, the transformation takes transitions as states and attach each possible reward to a state, in order to preserve the reward sequence. The labels along transitions denote $p^{\dagger}_{\pi}(x^{\dagger},y^{\dagger})$, and the labels $\boxed{r^{\dagger}(x^{\dagger})}$ near states denote the state-based reward function from the transformation.}
	\label{tMrp}
\end{figure*}

\subsection{SAT for Case 3}
We give the transformation theorem for MDPs with a $ r_{ST} $ (Equation~(\ref{rewardTypeST})) in a control setting and prove it as follows.

\begin{theorem} [Transformation for MDPs]
	Given an MDP $\langle S, A, r, p, \mu \rangle$ with $r$ stochastic and transition-based, there exists an MDP $\langle S^{\dagger}, A, r^{\dagger}, p^{\dagger}, \mu^{\dagger} \rangle$ with $r^{\dagger}$ deterministic and state-based, such that for any given policy (possibly randomized) for $\langle S, A, r, p, \mu \rangle$, there exists a corresponding policy for $\langle S^{\dagger}, A, r^{\dagger}, p^{\dagger}, \mu^{\dagger} \rangle$, such that both Markov reward processes share the same return distribution.
	\label{TransMDP}
\end{theorem}

\begin{proof}
	The proof has two steps. Step 1 constructs a second MDP and shows that, for every possible sample path in the first MDP, there exists a corresponding sample path in the second MDP. Step 2 proves that, the probability of any possible sample path in first MDP equals to the probability of its counterpart in the second MDP.
	
	\textbf{Step 1}:	
	Define $ S = \{1, \cdots, |S|\} $.
	For $ x,y \in S, a \in A_x$, define $ J = \supp(r(\cdot \mid x,a,y)) $. 
	Define $S^{\ddagger} = S^2 \times A \times J$.
	In order to remove the dependency of the initial state distribution on policy, define a null state space $ W = \{w_1, \cdots, w_{|S|}\} $, and $ W \cap S^{\ddagger} = \emptyset $.
	Define the state space $ S^{\dagger} = S^{\ddagger} \cup W $. 
	
	For all $x^{\dagger} = (x,a,y,j), y^{\dagger} = (y,a',y',j') \in S^{\dagger}, \text{ where } y' \in S, a' \in A_y, j,j' \in J$, define the state-based reward function $r^{\dagger}(j \mid x^{\dagger}, \cdot) = 1$, and $r^{\dagger}(0 \mid w_{\cdot}, \cdot) = 1$; 
	define the transition kernel $p^{\dagger}(y^{\dagger}  \mid  x^{\dagger}, a') = p(y'  \mid  y,a') r(j' \mid y,a',y')$, and $p^{\dagger}(x^{\dagger}  \mid  w_x, a) = p(y  \mid  x,a) r(j \mid x,a,y) $;
	the initial state distribution $\mu^{\dagger}(w_x) = \mu(x)$. 
	
	Now we have two MDPs. Let $MDP_1 = \langle S, A, r, p, \mu \rangle$, and $MDP_2 = \langle S^{\dagger}, A, r^{\dagger}, p^{\dagger}, \mu^{\dagger} \rangle$. 
	For any sample path $ (x_1, a_1, j_1, x_2, a_2, j_2, x_3, a_3, j_3, x_4\cdots) $ in $ MDP_1 $,
	there exists a sample path 
	\begin{gather*}
	    (w_{x_1}, a_1, 0, (x_1, a_1, j_1, x_2), a_2, j_1, (x_2, a_2, j_2, x_3), a_3,\\
	     j_2, (x_3, a_3, j_3, x_4), \cdots)
	\end{gather*}
	in $ MDP_2 $.
	Therefore, we proved that for every possible sample path for the first MDP, there exists a corresponding sample path for the second MDP. 
	
	\textbf{Step 2}:	
	Next we prove the probabilities for the two sample paths are equal.
	Here we prove it by mathematical induction. Set time epoch to $ n $ after the first $ x_{n+1} $ in $ MDP_1 $, and after the first $ (x_n, a_n, j_n, x_{n+1}) $ in $ MDP_2 $.
	
	Denote the partial sample path till epoch $ i $ by $ \iota_i $ in $ MDP_1 $. Given any $ \pi \in \Pi_R $ for $ MDP_1 $, the probability for the sample path before epoch $ 1 $ in $ MDP_1 $ is 
		\begin{gather*}
		    \mathbb{P}(\iota_1 = (x_1, a_1, j_1, x_2)) =\\ \mu(x_1) \pi(a_1 \mid x_1) p(x_2 \mid x_1,a_1) r(j_1 \mid x_1,a_1,x_2).
		\end{gather*}
	There exists a policy $ \pi^{\dagger} $ for $ MDP_2 $, with $ \pi^{\dagger}(a \mid w_x) = \pi(a \mid x) $, and $ \pi^{\dagger}(a \mid x^{\dagger}) = \pi(a \mid y) $.
	The probability for the sample path before epoch $ 1 $ in $ MDP_2 $ is 
	\begin{align*} 
	&\mathbb{P}(\iota^{\dagger}_1 = (w_{x_1}, a_1, 0, (x_1, a_1, j_1, x_2)))\\
	 &= \mu(w_{x_1}) \pi^{\dagger}(a_1 \mid w_{x_1}) p((x_1, a_1, j_1, x_2) \mid w_{x_1},a_1) \\
	&= \mathbb{P}(\iota_1 = (x_1, a_1, j_1, x_2)).
	\end{align*} 
	Therefore, at epoch $ 1 $, the two partial sample paths share the same probability. 
	
	Assuming that the two partial sample paths share the same probability at epoch $ n $, then the probability for the sample path before epoch $ n+1 $ in $ MDP_1 $ is
	\begin{align*}
	&\mathbb{P}(\iota_{n+1} = (\iota_n,x_{n+1}, a_{n+1}, j_{n+1}, x_{n+2})) \\
	&= \mathbb{P}(\iota_n = (x_1, \cdots ,x_n, a_n, j_n, x_{n+1})) \pi(a_{n+1} \mid x_{n+1}) \times \\ 
	&p(x_{n+2} \mid x_{n+1},a_{n+1}) r(j_{n+1} \mid x_{n+1},a_{n+1},x_{n+2}).
	\end{align*}
	The probability for the sample path before epoch $ n+1 $ in $ MDP_2 $ is 
	\begin{multline*}
	\mathbb{P}(\iota^{\dagger}_{n+1} = (\iota^{\dagger}_n, a_{n+1}, j_n, (x_{n+1}, a_{n+1}, j_{n+1}, x_{n+2}))) = \\ \mathbb{P}(\iota^{\dagger}_n = (w_{x_1}, \cdots ,(x_n, a_n, j_n, x_{n+1}))) \times \\ \pi^{\dagger}(a_{n+1} \mid (x_n, a_n, j_n, x_{n+1})) \times \\ 
	p((x_{n+1}, a_{n+1}, j_{n+1}, x_{n+2}) \mid (x_n, a_n, j_n, x_{n+1}),a_{n+1}) \times \\
	r(j_n \mid (x_n, a_n, j_n, x_{n+1}),a_{n+1}) \\ = \mathbb{P}(\iota_{n+1} = (\iota_n,x_{n+1}, a_{n+1}, j_{n+1}, x_{n+2})).
	\end{multline*}
	
	By induction we proved that, the probability of any possible sample path in $\langle S, A, r, p, \mu \rangle$ equals to the probability of its counterpart in $\langle S^{\dagger}, A, r^{\dagger}, p^{\dagger}, \mu^{\dagger} \rangle$. 
	As a subsequence of a sample path, the reward sequence $ \{R_t\} $ is preserved with a null reward at $ t=1 $. 
	Given the discount factor $ \gamma $, we can compensate this time drift effect simply by setting $ r^{\dagger}(x^{\dagger}) = r^{\dagger}(x^{\dagger})/\gamma$ to preserve the return distribution. 
	Theorem~\ref{TransMDP} is proved.
	
\end{proof}

It is worth noting that the SATs also works for some risk measures that are not law invariant, such as dynamic risk measures~\cite{ruszczynski2010risk}, since $ \{R_t\} $ can be preserved as well. It should also be pointed out that, since the state space is augmented, the SATs has an effect on the computational complexity. Denote the complexity for an algorithm by $ T(|S|, |A|) $, it becomes $ T(|S^2 \times A \times J| + |S|, |A|) $ when the SAT for Case 3 is implemented.

\subsection{SAT for Case 2 and 1}
Here, we present SATs for MDPs in policy evaluation settings.
Given an MDP with a randomized policy $ \pi \in \Pi_R $, the reward function is often simplified as well. Taking a deterministic state-based reward function $ r_{DS} $ for example, the reward function is simplified to $ \sum_{y \in S} \pi(a|x) r_{DS}(x,a), x,y \in S, a \in A_x$.
In order to preserve the reward sequence, 
one way is to consider action in a ``situation'' in Remark~\ref{remark1}. 
\begin{remark}[State augmentation in the transformations]
	In order to transform a Markov reward process with a $ r_{ST}$ 
	(or $ r_{SS}$, $ r_{DT}$) to the one with a $ r_{DS} $, and preserve the return sequence at the same time, a bijective mapping between a new ``augmented'' state space and the possible ``situation'' space is needed. For a Markov reward process with $ r_{ST}$, a possible situation can be defined by a tuple $ \langle x, y, j \rangle $, in which $x,y \in S, j \in \supp(r_{ST}(\cdot \mid x,y))$.
	\label{remark1}
\end{remark}

Next, we present the following corollaries for Case 2 and Case 1.

\begin{corollary} [Transformation for MDP with a randomized policy]
	For a Markov decision process $\langle S, A, r, p, \mu \rangle$ with $ r $ stochastic and transition-based, and a randomized policy $ \pi \in \Pi_R $, there exists a Markov reward process $\langle S^{\dagger}, r^{\dagger}_{\pi}, p^{\dagger}_{\pi}, \mu^{\dagger}_{\pi} \rangle$ with $r^{\dagger}_{\pi}$ deterministic and state-based, such that both processes share the same return sequence.
	\label{tC2}
\end{corollary}

Similarly, we present a corollary for Case 1, which is a special case of Corollary~\ref{tC2}.

\begin{corollary} [Transformation for Markov process with a stochastic transition-based reward function]
	For a Markov reward process $\langle S, r_{\pi}, p_{\pi}, \mu_{\pi} \rangle$ with $ r_{\pi} $ stochastic and transition-based, there exists a Markov reward process $\langle S^{\dagger}, r^{\dagger}_{\pi}, p^{\dagger}_{\pi}, \mu^{\dagger}_{\pi} \rangle$ with $r^{\dagger}_{\pi}$ deterministic and state-based, such that both processes share the same return sequence.
	\label{tC1}
\end{corollary}

\subsection{SAT for Case 0}
When we use Q-learning, or other Q-function (value function) based learning methods in an MDP with a reward function which is not deterministic and state-based, it implies a reward simplification similar to Equation~(\ref{simplifiedR}).
This simplification changes the reward distribution as well as a risk measure. 
In this subsection, we show the effect on the distribution from the reward simplification. We apply the state-transition transformation~\cite{Ma2017STT}---which is the SAT for Case 0---in an infinite-horizon with a discount factor setting, estimate the reward distribution and the VaR result, and compare them with the ones from reward simplification.

Figure~\ref{MdpDesc} describes an MDP with transition-based reward function and its counterpart with state-based reward function following reward simplication.
In this MDP, consider a deterministic policy $ \pi = [2,1,0] $---to order 2, 1, 0 item(s) when the inventory level is 0, 1, 2, respectively---then we have a Markov reward process illustrated in Figure~\ref{tMrp}(a). 
In order to attach each possible reward value to a state to construct a $ r_{DS} $, we consider each transition as a state, and apply the \textit{State-Transition Transformation} (Algorithm~\ref{STT}~\cite{Ma2017STT}). 
The transformed Markov reward process is presented in Figure~\ref{tMrp}(b), where only some of the transitions are shown.
Given the same Markov reward process without the transformation, the reward function is often simplified by Equation~(\ref{simplifiedR}), which only preserves the expectation. 

\begin{algorithm}[t]
   \caption{State-Transition Transformation (for Case 0)}
   \label{STT}
\begin{algorithmic}
   \STATE {\bfseries Input:} Markov reward process $\langle S, r_{\pi}, p_{\pi}, \mu\rangle$.
   \STATE {\bfseries Output:} Markov reward process $\langle S^{\dagger}, r^{\dagger}_{\pi}, p^{\dagger}_{\pi}, \mu^{\dagger}\rangle$.
   \STATE Generate the state space $S^{\dagger} = S \times S$;
   \FOR{{\bfseries all} $x^{\dagger} = (x,y)$ {\bfseries where} $x,y \in S$}
   \STATE Construct the reward function $r^{\dagger}_{\pi}(x^{\dagger}) = r_{\pi}(x,y)$;
   \STATE Construct the transition kernel \\
      		$p^{\dagger}_{\pi}(x^{\dagger}  \mid  y^{\dagger}) = p_{\pi}(y  \mid  x)$ for all $y^{\dagger}=(\cdot,x) \in S^{\dagger}$, and $p^{\dagger}_{\pi}(x^{\dagger}  \mid  y^{\dagger}) = 0$ otherwise;
   \STATE Set the initial state distribution \\ $\mu^{\dagger}(x^{\dagger}) = \mu(x)p_{\pi}(y  \mid  x)$;
   \ENDFOR
\end{algorithmic}
\end{algorithm}
\begin{figure}[b!]
	\centering
	\centerline{\includegraphics[scale=0.30]{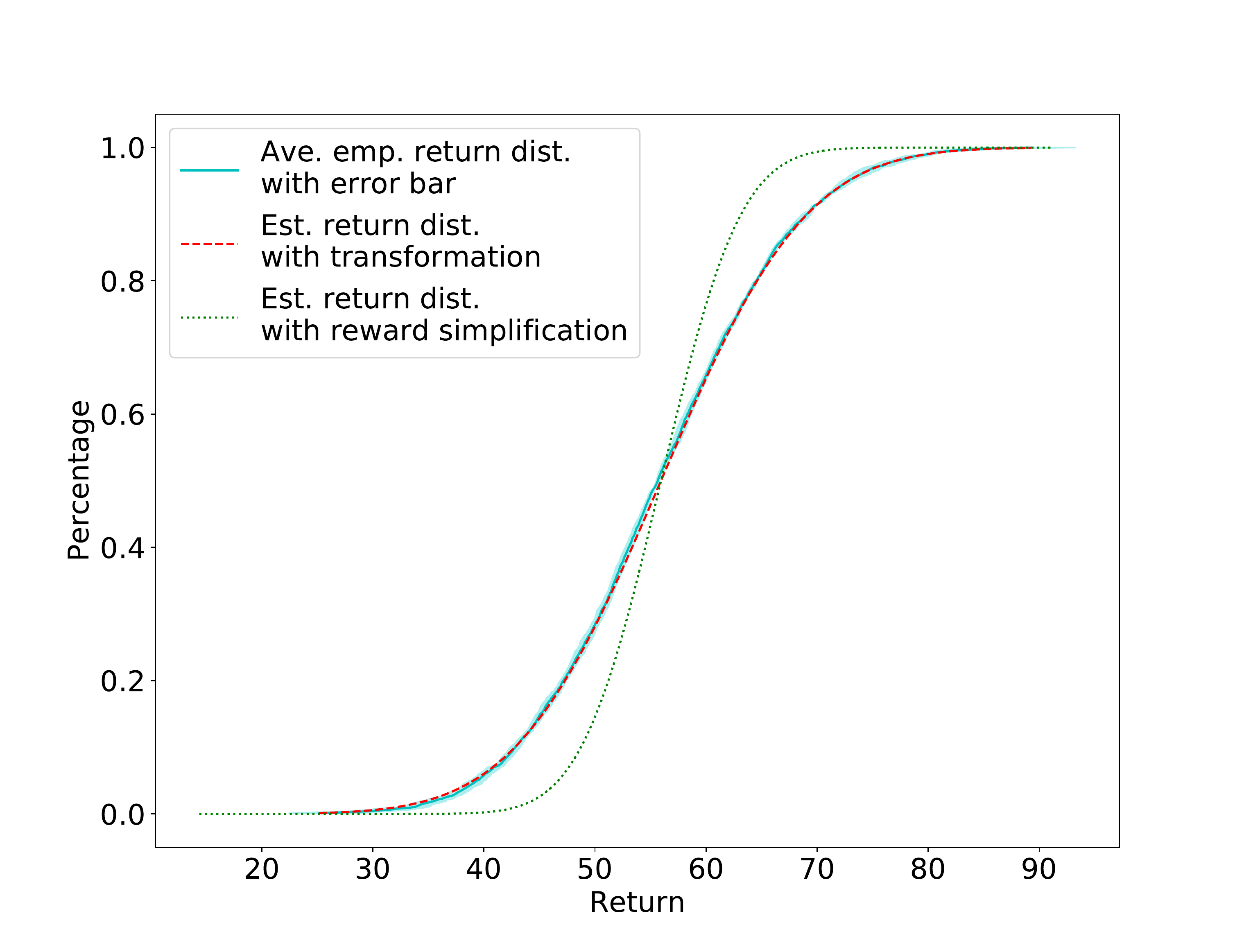}}
	\caption{Comparison among the averaged empirical return distribution with error region, the estimated return distribution from the transformation, and the estimated return distribution from the reward simplification.}
	\label{dist3errorbar}
\end{figure}
Now we set $ \gamma = 0.95 $ and compare the two return distributions---one from the transformation and the other from the reward simplification---with the averaged empirical return distribution.
Since there is no central limit theorem for discounted Markov processes, to simplify the estimation, we estimate the return distribution by involving mean and variance only, which implies the assumption that the return is approximately normal.
With the aids of Theorem~\ref{theorem1982}, the two distributions are shown in Figure~\ref{dist3errorbar}.
The averaged empirical return distribution is from a simulation repeated 50 times with a time horizon 1000, with the error region representing the standard deviations of the means along return axis. 
The Kolmogorov--Smirnov statistic $ D_{KS} $~\cite{durbin1973distribution} is used to quantify the distribution difference (error). Denote the averaged empirical return distribution by $ \hat{F}^{\pi}_{\Phi} $, the estimated distribution for the transformed process by $ Q^{\pi}_{\Phi} $, and the estimated distribution for the process with the reward simplification by $ Q'^{\pi}_{\Phi} $. For the case in Figure~\ref{dist3errorbar}, $ D_{KS} (Q'^{\pi}_{\Phi}, \hat{F}^{\pi}_{\Phi}) = \sup_{\phi \in \mathbb{R}} |Q'^{\pi}_{\Phi}(\phi) - \hat{F}^{\pi}_{\Phi}(\phi)| \approx 0.145$, and $ D_{KS} (Q^{\pi}_{\Phi}, \hat{F}^{\pi}_{\Phi}) \approx 0.012$. The results show that, the reward simplification leads to a nontrivial estimation error ($ 0.145 > 0.012 $).
Notice that the direct use of Q-learning will result in a risk measure on a learned distribution, which is supposed to converge to the estimated distribution with the reward simplification.

\begin{figure}[b!]
	\centering
	\centerline{\includegraphics[scale=0.30]{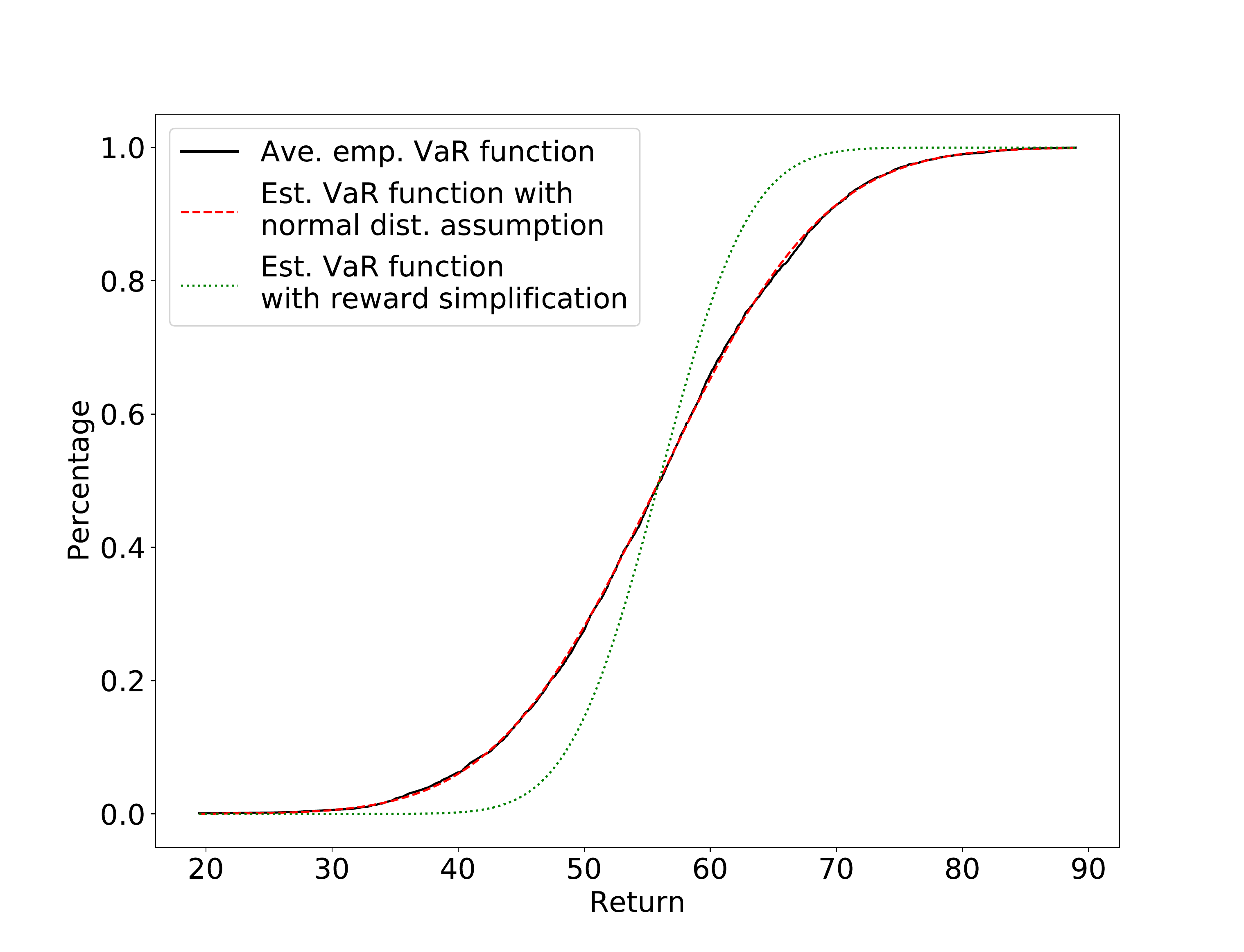}}
	\caption{Comparison among the averaged empirical VaR function, the estimated VaR function from the transformation, and the estimated VaR function from the reward simplification.}
	\label{VaR3}
\end{figure}

Now consider the VaR objective (VaR Definition~\ref{VaR2} for example). The VaR function is obtained by enumerating the deterministic policies. Figure~\ref{VaR3} shows the two estimated VaR functions. Since the VaR function can also be regarded as a return distribution, we can still use $ D_{KS} $ to measure the error from the reward simplification, and in this case $ D_{KS} \approx 0.150 $. Denote the optimal quantile for the MDP with the reward simplification by $ \hat{\eta}_{\tau} $, then the error bound for the optimal quantile $ \sup\{|\hat{\eta}_{\tau} - \eta_{\tau}|:\tau \in \mathbb{R}\} \approx 0.150$, which is nontrivial in this risk-sensitive problem.

\begin{remark}[Smaller variance from reward simplification]
	In Figure~\ref{dist3errorbar} we can tell that the distribution for the process with the  reward simplification has a smaller variance. The reason can be intuitively explained by the analysis of variance~\cite{scheffe1999analysis}. Taking the deterministic transition-based reward function for example. Considering the possible rewards for the same current state as a group, the variance for $ R_t $ includes the variances between groups and the variances within groups. When the reward function is simplified with Equation~(\ref{simplifiedR}), the variances within groups are removed, so the variance is smaller. 
\end{remark}

From the inventory example we can see that, some methods require Markov processes to be with deterministic and state-based reward functions only, and the reward simplification changes the reward sequences (distributions).
If we want to use Q-learning, or other methods for Markov processes with $ r_{DS} $ in a risk-sensitive scenario, we should implement an appropriate SAT first.

\section{Related Works}
The risk concerns arise in RL in two aspects. One refers to the ``external'' uncertainty in the model parameters, and this problem is known as robust MDPs. In the robust MDPs people optimize the expected return with the worst-case parameters, which belongs to a set of plausible MDP parameters. For example, an MDP with uncertain transition matrices~\cite{nilim2005robust}.
The other one refers to the ``internal'' risk, which studies the stochastic property of the process. 
In general, there are three internal risk-sensitive objective classes in RL area, which have been studied for decades. 
One is the mean-variance risk measure~\cite{D.J.White1988a,doi:10.1287/opre.42.1.175,Mannor2011a}, also known as the modern portfolio theory, in which the expected return is maximized with a given risk level (variance). 
The other is target-percentile risk measure~\cite{Filar1995b,Wu1999a}. 
This risk measure formulate the objective in terms of the probabilities of certain targets (or quantiles), and optimize, for example, the expected return. 
Another is utility risk measure, in which the exponential utility function is concerned.
The internal risk concerns arise not only mathematically but also psychologically. 
A classic example in psychology is the ``St. Petersburg Paradox,'' which refers to a lottery with an infinite expected reward, but people only prefer to pay a small amount to play. 
This problem is thoroughly studied in utility theory, and a recent study brought this idea to RL~\cite{Prashanth2015}.
Some risk measures are coherent~\cite{Artzner1998a}, which share some intuitively reasonable properties (convexity, for example). Ruszczy{\'{n}}ski and Shapiro~\shortcite{Ruszczynski2006a} presented a thorough study on coherent risk optimization. 

 
Q-learning has been studied in risk-sensitive RL for decades.
Many risk-sensitive Q-learning studies are for MDPs with an $ r_{DS} $.
Borkar~\shortcite{borkar2002q} proved the convergence of Q-learning for an exponential utility cost objective with an ordinary differential equation method.
A trajectory-based algorithm which combines policy gradient and actor-critic methods was presented to solve a CVaR-constrained problem~\cite{chow2017risk}.
For robust MDP problems, with considering a set of general uncertainties (random action, unknown cost and safety hazards), an approach was provided to compute safe and optimal strategies iteratively~\cite{junges2016safety}.
Q-learning has also been used to provide risk-sensitive analysis on the fMRI signals, which provides a better interpretation of the human behavior in a sequential decision task~\cite{shen2014risk}.
The expectation-based worst case risk measure might not need the proposed SATs. For example, the minimax risk measures studied in~\cite{huang2017risk}. 
For a comprehensive survey on safe RL, see~\cite{garcia2015comprehensive}.

\section{Conclusion}
The proposed SATs transform MDPs and Markov reward processes with stochastic transition-based reward functions into ones with deterministic state-based reward functions, and preserve the reward sequences (and distributions) for risk-sensitive objectives.
In an infinite-horizon time-homogeneous MDP for an inventory control problem, we illustrate the error on the distribution from the reward simplification.
Taking the advantage of the variance formula presented by Sobel~\shortcite{sobel1982variance}, we estimate the return distribution for a Markov reward process. 
Since many RL methods require the reward function to be deterministic and state-based, the transformation is needed for the MDPs with other types of reward functions in the risk-sensitive problems.
We generalize the transformation~\cite{Ma2017STT} in different settings, and consider VaR as an example to show the effect of reward simplification on distribution. 

In many practical problems, the rewards are simplified to deterministic and state-based.
When the MDP is with a reward function which is not an $ r_{DS} $ type, the direct use of Q-learning also implies such a reward simplification.
In this paper, the error from the reward simplification on distribution is illustrated, which is crucial in risk-sensitive cases. 
By implementing the transformation instead of the reward simplification, the MDPs with different types of reward functions and (or) randomized policies are transformed to the ones with deterministic and state-based reward functions with an intact reward distribution. 

The essence of the SATs is to attach each possible reward value to an augmented state to preserve the reward sequence. 
This attachment is crucial in risk-sensitive reinforcement learning, considering the wide application of value function and Q-function.
Without the proposed SATs, the learned Q-function (value function) only estimates the expected value of a given state-action pair (state) when the reward function is (stochastic) transition-based.
Now with the SATs, the Q-function (value function) can be considered as an approximation of the ``real'' value of a given augmented state-action pair (state).
In other words, the proposed transformations ``transform'' the uncertainties from the transition, action, and the stochasticity of the reward function to the augmented state space.
The proposed SATs present a platform for Q-learning in risk-sensitive RL, and we believe that many related studies should be revisited with the proposed SATs instead of applying the common-used reward simplification directly.

\section{Acknowledgments}
This research was supported in part by the scholarship from China Scholarship Council (CSC), and the Natural Sciences and Engineering Research Council of Canada (NSERC) under Grants 509935 and 512046.
\bibliography{ref_ma}

\begin{thebibliography}{}

\bibitem[\protect\citeauthoryear{Altman}{1999}]{altman1999constrained}
Altman, E.
\newblock 1999.
\newblock {\em Constrained Markov Decision Processes}.
\newblock CRC Press.

\bibitem[\protect\citeauthoryear{Artzner \bgroup et al\mbox.\egroup
  }{1998}]{Artzner1998a}
Artzner, P.; Delbaen, F.; Eber, J.; and Heath, D.
\newblock 1998.
\newblock {Coherent measures of risk}.
\newblock {\em Mathematical Finance} 9(3):1--24.

\bibitem[\protect\citeauthoryear{Berkenkamp \bgroup et al\mbox.\egroup
  }{2017}]{berkenkamp2017safe}
Berkenkamp, F.; Turchetta, M.; Schoellig, A.; and Krause, A.
\newblock 2017.
\newblock Safe model-based reinforcement learning with stability guarantees.
\newblock In {\em Proceedings of the 31st Advances in Neural Information
  Processing Systems (NIPS)},  908--918.

\bibitem[\protect\citeauthoryear{Borkar}{2002}]{borkar2002q}
Borkar, V.~S.
\newblock 2002.
\newblock Q-learning for risk-sensitive control.
\newblock {\em Mathematics of Operations Research} 27(2):294--311.

\bibitem[\protect\citeauthoryear{Chow \bgroup et al\mbox.\egroup
  }{2017}]{chow2017risk}
Chow, Y.; Ghavamzadeh, M.; Janson, L.; and Pavone, M.
\newblock 2017.
\newblock Risk-constrained reinforcement learning with percentile risk
  criteria.
\newblock {\em The Journal of Machine Learning Research} 18(1):6070--6120.

\bibitem[\protect\citeauthoryear{Durbin}{1973}]{durbin1973distribution}
Durbin, J.
\newblock 1973.
\newblock {\em Distribution Theory for Tests based on the Sample Distribution
  Function}.
\newblock SIAM.

\bibitem[\protect\citeauthoryear{Filar \bgroup et al\mbox.\egroup
  }{1995}]{Filar1995b}
Filar, J.~A.; Krass, D.; Ross, K.~W.; and Member, S.
\newblock 1995.
\newblock Percentile performance criteria for limiting average {Markov}
  decision processes.
\newblock {\em IEEE Transactions on Automatic Control} 40(I):2--10.

\bibitem[\protect\citeauthoryear{Garc{\'i}a and
  Fern{\'a}ndez}{2015}]{garcia2015comprehensive}
Garc{\'i}a, J., and Fern{\'a}ndez, F.
\newblock 2015.
\newblock A comprehensive survey on safe reinforcement learning.
\newblock {\em Journal of Machine Learning Research} 16(1):1437--1480.

\bibitem[\protect\citeauthoryear{Gilbert and Weng}{2016}]{gilbert2016quantile}
Gilbert, H., and Weng, P.
\newblock 2016.
\newblock Quantile reinforcement learning.
\newblock {\em arXiv:1611.00862}.

\bibitem[\protect\citeauthoryear{Huang and Haskell}{2017}]{huang2017risk}
Huang, W., and Haskell, W.~B.
\newblock 2017.
\newblock {Risk-aware Q-learning for Markov decision processes}.
\newblock In {\em Proceedings of the 56th IEEE Conference on Decision and
  Control (CDC)},  4928--4933.

\bibitem[\protect\citeauthoryear{Junges \bgroup et al\mbox.\egroup
  }{2016}]{junges2016safety}
Junges, S.; Jansen, N.; Dehnert, C.; Topcu, U.; and Katoen, J.-P.
\newblock 2016.
\newblock Safety-constrained reinforcement learning for {MDPs}.
\newblock In {\em Proceedings of the 22nd International Conference on Tools and
  Algorithms for the Construction and Analysis of Systems (TACAS)},  130--146.
\newblock Springer.

\bibitem[\protect\citeauthoryear{Kusuoka}{2001}]{kusuoka2001law}
Kusuoka, S.
\newblock 2001.
\newblock On law invariant coherent risk measures.
\newblock In {\em Advances in Mathematical Economics}. Springer.
\newblock  83--95.

\bibitem[\protect\citeauthoryear{Ma and Yu}{2017}]{Ma2017STT}
Ma, S., and Yu, J.~Y.
\newblock 2017.
\newblock Transition-based versus state-based reward functions for {MDPs} with
  {Value-at-Risk}.
\newblock In {\em Proceedings of the 55th Annual Allerton Conference on
  Communication, Control, and Computing (Allerton)},  974--981.

\bibitem[\protect\citeauthoryear{Mannor and Tsitsiklis}{2011}]{Mannor2011a}
Mannor, S., and Tsitsiklis, J.
\newblock 2011.
\newblock Mean-variance optimization in {Markov} decision processes.
\newblock In {\em Proceedings of the 28th International Conference on Machine
  Learning (ICML)},  1--22.

\bibitem[\protect\citeauthoryear{Meyn and Tweedie}{2009}]{meyn2012markov}
Meyn, S.~P., and Tweedie, R.~L.
\newblock 2009.
\newblock {\em Markov Chains and Stochastic Stability}.
\newblock Springer Science \& Business Media.

\bibitem[\protect\citeauthoryear{Nilim and Ghaoui}{2005}]{nilim2005robust}
Nilim, A., and Ghaoui, L.~E.
\newblock 2005.
\newblock Robust control of {Markov} decision processes with uncertain
  transition matrices.
\newblock {\em Operations Research} 53(5):780--798.

\bibitem[\protect\citeauthoryear{Prashanth \bgroup et al\mbox.\egroup
  }{2016}]{Prashanth2015}
Prashanth, L.~A.; Jie, C.; Fu, M.; Marcus, S.; and Szepesv{\'a}ri, C.
\newblock 2016.
\newblock Cumulative prospect theory meets reinforcement learning: Prediction
  and control.
\newblock In {\em Proceedings of the 33rd International Conference on Machine
  Learning (ICML)},  1406--1415.

\bibitem[\protect\citeauthoryear{Puterman}{1994}]{Puterman1994a}
Puterman, M.~L.
\newblock 1994.
\newblock {\em {Markov Decision Processes: Discrete Stochastic Dynamic
  Programming}}.
\newblock Wiley.

\bibitem[\protect\citeauthoryear{Riedel}{2004}]{Riedel2004}
Riedel, F.
\newblock 2004.
\newblock {Dynamic coherent risk measures}.
\newblock {\em Stochastic Processes and their Applications} 112(2):185--200.

\bibitem[\protect\citeauthoryear{Ruszczy{\'{n}}ski and
  Shapiro}{2006}]{Ruszczynski2006a}
Ruszczy{\'{n}}ski, A., and Shapiro, A.
\newblock 2006.
\newblock Optimization of convex risk functions.
\newblock {\em Mathematics of Operations Research} 31(3):433--452.

\bibitem[\protect\citeauthoryear{Ruszczy{\'{n}}ski}{2010}]{ruszczynski2010risk}
Ruszczy{\'{n}}ski, A.
\newblock 2010.
\newblock Risk-averse dynamic programming for {Markov} decision processes.
\newblock {\em Mathematical Programming} 125(2):235--261.

\bibitem[\protect\citeauthoryear{Scheff{\'{e}}}{1999}]{scheffe1999analysis}
Scheff{\'{e}}, H.
\newblock 1999.
\newblock {\em The Analysis of Variance}.
\newblock John Wiley \& Sons.

\bibitem[\protect\citeauthoryear{Shen \bgroup et al\mbox.\egroup
  }{2014}]{shen2014risk}
Shen, Y.; Tobia, M.~J.; Sommer, T.; and Obermayer, K.
\newblock 2014.
\newblock Risk-sensitive reinforcement learning.
\newblock {\em Neural Computation} 26(7):1298--1328.

\bibitem[\protect\citeauthoryear{Sobel}{1982}]{sobel1982variance}
Sobel, M.~J.
\newblock 1982.
\newblock The variance of discounted {Markov} decision processes.
\newblock {\em Journal of Applied Probability} 19(4):794--802.

\bibitem[\protect\citeauthoryear{Sobel}{1994}]{doi:10.1287/opre.42.1.175}
Sobel, M.~J.
\newblock 1994.
\newblock Mean-variance tradeoffs in an undiscounted {MDP}.
\newblock {\em Operations Research} 42(1):175--183.

\bibitem[\protect\citeauthoryear{White}{1988}]{D.J.White1988a}
White, D.~J.
\newblock 1988.
\newblock Mean , variance , and probabilistic criteria in finite {Markov}
  decision processes : A review.
\newblock {\em Journal of Optimization Theory and Applications} 56(1):1--29.

\bibitem[\protect\citeauthoryear{Woodroofe}{1992}]{woodroofe1992central}
Woodroofe, M.
\newblock 1992.
\newblock A central limit theorem for functions of a {Markov} chain with
  applications to shifts.
\newblock {\em Stochastic processes and their Applications} 41(1):33--44.

\bibitem[\protect\citeauthoryear{Wu and Lin}{1999}]{Wu1999a}
Wu, C., and Lin, Y.
\newblock 1999.
\newblock {Minimizing risk models in Markov decision process with policies
  depending on target values}.
\newblock {\em Journal of Mathematical Analysis and Applications} 23(1):47--67.

\end{thebibliography}
\bibliographystyle{aaai}
\end{document}